\documentclass[10pt,twocolumn,letterpaper]{article}

\usepackage{iccv}
\usepackage{times}
\usepackage{epsfig}
\usepackage{graphicx}
\usepackage{amsmath}
\usepackage{amssymb}
\usepackage{times}
\usepackage{epsfig}
\usepackage{graphicx}
\usepackage{amsmath}
\usepackage{amssymb}
\usepackage[]{algorithm2e}
\usepackage{makecell}
\usepackage{amsthm}
\usepackage{soul}
\usepackage{multirow}

\newtheorem{theorem}{Theorem}[section]

\newtheorem{prop}[theorem]{Proposition}
\newtheorem{corr}[theorem]{Corollary}

\theoremstyle{definition}
\newtheorem{definition}[theorem]{Definition}
\theoremstyle{definition}

\theoremstyle{remark}

\theoremstyle{remark}

\theoremstyle{remark}

\newcommand*{\QEDbs}{\hfill\ensuremath{\blacksquare}}%
\usepackage[breaklinks=true,bookmarks=false]{hyperref}

\iccvfinalcopy 

\def\iccvPaperID{****} 

\begin{document}


\title{Truncating Wide Networks using Binary Tree Architectures}

\author{Yan Zhang$^1$, Mete Ozay$^1$, Shuohao Li$^{1,2}$ and Takayuki Okatani$^1$\\
$^1$Tohoku University, Sendai Japan\\
$^2$National University of Defense Technology, Changsha China\\
{\tt\small \{zhang,mozay,lishuohao,okatani\}@vision.is.tohoku.ac.jp}
}

\maketitle

\begin{abstract}
Recent study shows that a wide deep network can obtain  accuracy comparable to a deeper but narrower network. Compared to narrower and deeper networks, wide networks employ relatively less number of layers and have various important benefits, such that they have less running time on parallel computing devices, and they are less affected by gradient vanishing problems. 
However, the parameter size of a wide network can be very large due to use of large width of each layer in the network. 
In order to keep the benefits of wide networks meanwhile improve the parameter size and accuracy trade-off of wide networks, we propose a binary tree architecture to truncate architecture of wide networks by reducing the width of the networks. More precisely, in the proposed architecture, the width is continuously reduced from lower layers to higher layers in order to increase the expressive capacity of network with a less increase on parameter size. Also, to ease the gradient vanishing problem, features obtained at different layers are concatenated to form the output of our architecture.
By employing the proposed architecture on a baseline wide network, we can construct and train a new network with same depth but considerably less number of parameters. In our experimental analyses, we observe that the proposed architecture enables us to obtain better parameter size and accuracy trade-off compared to baseline networks using various benchmark image classification datasets. The results show that our model can decrease the classification error of baseline from $20.43\%$ to $19.22\%$ on Cifar-100 using only $28\%$ of parameters that baseline has. Code is available at {\color{red}  \url{https://github.com/ZhangVision/bitnet}}
\end{abstract}

\section{Introduction}

Recently, Deep Neural Networks (DNNs) have achieved impressive results for many image classification tasks~\cite{alexnet,vgg,NIN,googlenet,resnet,rethinking,resnetpreact,inceptionv4}.
Various architectures of DNNs have been proposed in order to improve classification accuracy. 
One approach used to improve accuracy of a DNN is to widen each layer while keeping the depth unchanged. For instance, it is empirically shown in~\cite{wrn} that a wide but relatively shallow DNN can obtain accuracy comparable to a narrow but relatively deep DNN on several classification tasks.
There are two crucial benefits of wide-shallow DNNs. First, it usually runs \textit{faster} than narrow-deep DNNs on parallel computing devices, e.g. GPUs, as illustrated in~\cite{wrn}. Also, a deep DNN with many layers may suffer from a \textit{gradient vanishing problem}. Reducing the depth can ease this problem as shown in~\cite{stochasticdepth}.
However, parameter size of DNNs may significantly increase with respect to improvement of accuracy by widening each layer. 


\begin{figure}[t]
\begin{center}
   \includegraphics[width=1\linewidth]{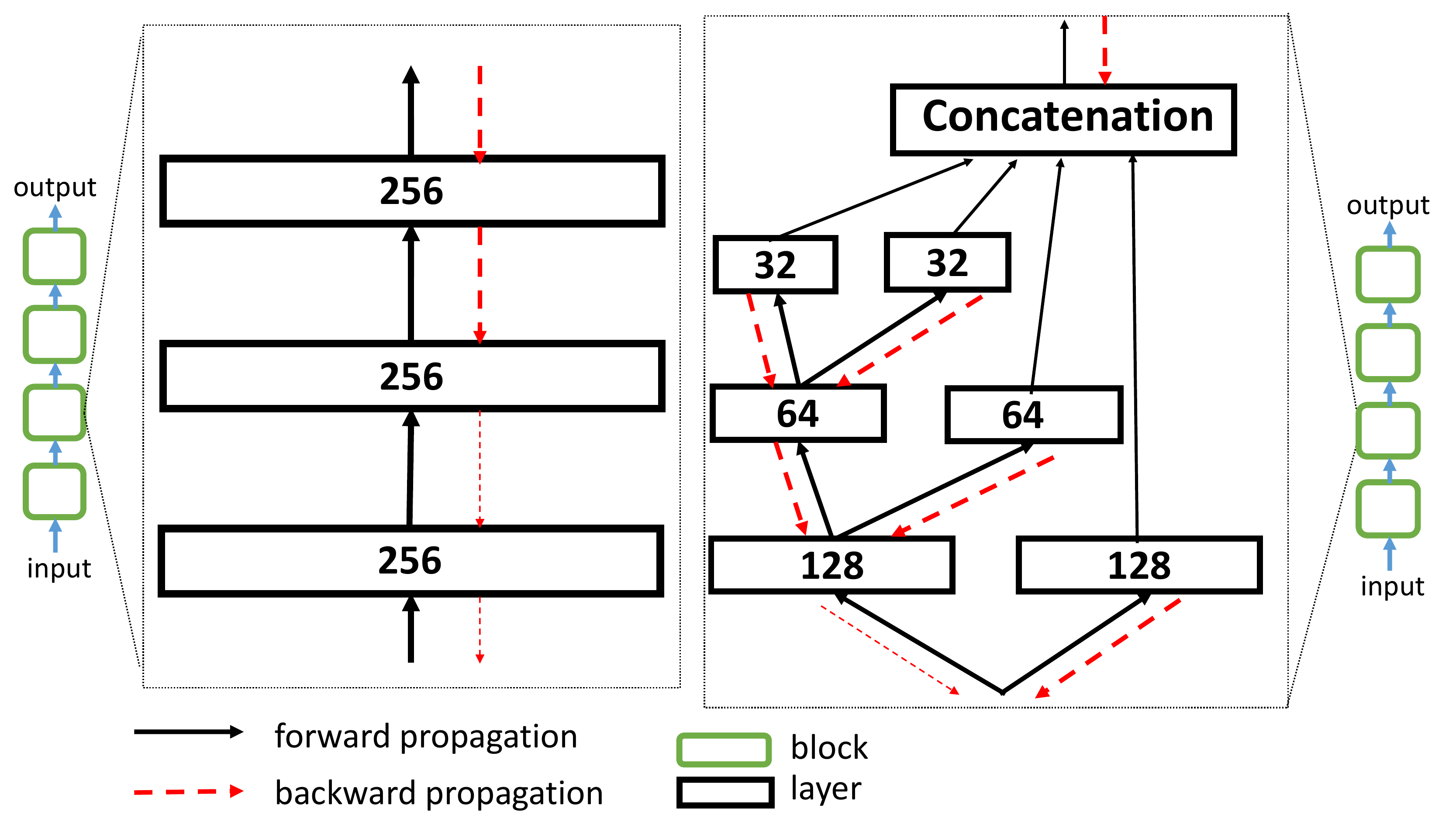}
\end{center}
   \caption{\textbf{Left}: A conventional architecture~\cite{resnet}. \textbf{Right}: Proposed binary tree architecture. Arrows indicate data flow. The numbers depicted in the rectangle denote the width of the layer. Both architectures have depth $3$ and output width $256$. Our architecture has considerably less number of parameters.}
\label{fig:bitnet}
\end{figure}


In this paper, we address a problem of improvement of the parameter size and accuracy trade-off of the aforementioned wide-shallow DNNs. Specifically, given a baseline wide-shallow DNN, we aim to construct and train a new DNN equipped with the following two desirable properties.
\begin{itemize}
    \item The new DNN has a depth not greater than the baseline wide-shallow DNN so that it can keep the aforementioned two benefits.
    
    \item The new DNN can achieve comparable accuracy using relatively less number of parameters compared to the baseline DNN, or can achieve better accuracy by using the same number of parameters as baseline DNN.
\end{itemize}

Toward this end, we propose a binary tree architecture for implementation of DNNs with a \textit{better} trade-off. An illustrative comparison of our proposed binary tree architecture and conventional architectures~\cite{resnet,wrn} is given in Figure~\ref{fig:bitnet}. In conventional architectures, layers having same width (number of channels) are sequentially stacked.
In the proposed binary tree architecture, the width of the $k^{th}$ layer is $\frac{D}{2^{k-1}}$, where $D$ is the width of the first layer ($k=1$) (input layer). Additionally, connections between layers of the proposed architecture are established as connections used in an asymmetric binary tree.
At each $k^{th}$ layer of a binary tree architecture, we have $C_k=\frac{D}{2^{k-1}}$ channels.
Then, $\frac{C_k}{2}$ of channels are connected to the channels of the $k+1^{st}$ layer. In addition, the remaining $\frac{C_k}{2}$ channels are directly concatenated to form the output of the architecture. Note that our binary tree architecture can be generalized to fully connected layers in which the width becomes the number of neurons used at the layer.

Our \textbf{motivation} for employment of the proposed binary tree architecture is twofold. First, we intend to increase the \textbf{expressive capacity of DNNs} with a relatively small increase of parameter size. In this paper, we use the definition of expressive capacity proposed in the previous work \cite{bengio1,bengio12}, where it is defined to be the maximal number of linear regions of (decision) functions computable by the given DNN. As shown in~\cite{bengio1,bengio12}, it reflects the complexity of class decision boundary computable by the DNN.
Their results state that the maximal number of linear regions
of a fully connected feed forward neural network endowed with ReLU~\cite{relu} activation functions grows exponentially with respect to the depth of the network, and polynomially with respect to the width of the network (i.e., the number of neurons used at each layer of the network). 
Following this theoretical result, one can increase the expressive capacity with a small increase of the parameter size by simply stacking more layers with \textit{small} width. This leads to the first characteristic structure of our binary tree architecture, which is obtained by continuous decrease of the width from input layer to higher layers by a factor of $2^{-1}$. 
With this specific structure, the expressive capacity grows with small increase of the parameter size. In our experiments, a binary tree used at convolutional layers of a DNN can increase classification accuracy with a small increase of  parameter size of the DNN.
%

The second motivation for employment of our binary tree architecture is to \textbf{ease a vanishing gradient problem observed in DNNs}. While training DNNs, the magnitude of gradient can be cumulatively reduced when it is propagated from higher layers to lower layers. Therefore, the more layers are used for propagation, the weaker gradient will be obtained at the lower layers, which makes it difficult to train a DNN with many layers~\cite{gradvanish2,gradvanish1}. 
Consequently, the gradient vanishing problem suggests reduction of the depth of a DNN, while increasing the depth may lead to efficient improvement of the expressive capacity as mentioned above. Thus, we need to trade-off between easing the gradient vanishing problem and increasing the expressive capacity, which  motivates our second characteristic structure of our proposed binary tree architecture that is obtained by concatenation of features obtained at different layers. With this specific structure, gradients can propagate through short path to lower layers. An illustration is given in Figure~\ref{fig:bitnet}, where flow of gradient propagation during backpropagation is depicted by red dash line.
For a better illustration of  vanishing gradients, we use thicker red line to show stronger gradient. Our empirical analyses in Section \ref{gv} show that concatenation of features at lower layers can ease the gradient vanishing problem.

Our contributions are summarized as follows. 

    (1) We propose a binary tree architecture to improve the trade-off between parameter size and classification accuracy of given baseline wide-shallow DNNs. Meanwhile, the depth of baseline wide-shallow DNNs is not increased to keep the benefit of running speed on parallel computing devices.
  
    (2) Our experimental results show that, on the Cifar datasets, one can construct a DNN using the proposed binary tree architecture to achieve better accuracy but using considerably less number of parameters. On the Cifar-100 dataset, our models can outperform corresponding baselines by using only approximately $50\%$ of baseline's parameter size. One of our models decreases the classification error of baseline from $20.43\%$ to $19.22\%$ by using only $28\%$ of parameters that baseline has.
On ILSVRC12 task, we construct and train two DNNs using the proposed binary tree architecture which also provide better parameter size and classification accuracy trade-off than baseline models. 

(3) We also provide a theoretical analysis of the expressive capacity of DNNs endowed with our proposed binary tree architecture as a function of its depth and width. The theoretical results indicate that expressive capacity of DNNs endowed with our proposed binary tree architecture can grow with small increase of the parameter size.

The rest of this paper is organized as follows. The second section provides the related work. In Section 3, we introduce our binary tree architecture. Experimental results and analyses are given in Section 4. Section 5 concludes the paper.

\section{Related Work}

 

Recently,
various architectures of convolutional neural networks (CNNs) have been proposed~\cite{alexnet,NIN,googlenet,vgg,rethinking,inceptionv4}. 
In~\cite{forest}, connections between random forest~\cite{rf} and CNNs are investigated. Inspired by random forest, they embedded routing functions to CNNs and obtain Conditional CNNs. As shown in their experiments, Conditional CNNs with highly branched tree architectures can improve the accuracy-efficiency trade-off. 
Conditional CNNs can be considered as symmetric full tree architectures. On the other hand, we use an asymmetric tree architecture and concatenate features from different layers.  
Another related work, a \textit{fractal architecture} used by FractalNet was proposed in~\cite{fractalnet}.
Fractal architecture can also be considered as a tree architecture.
Fractal architecture is different from ours in several aspects. First, the output of a fractal architecture is element-wise mean of features obtained at different layers. Also, all convolutional layers used in fractal architecture use the same width, which may result in a large parameter size if the width is large.

As shown in~\cite{speed_up,Efficient}, the parameter size of a trained CNN can be reduced by constructing a new CNN with less redundancy of weights. However, these methods may cause a drop on accuracy after a compression of model. With our architecture, we can boost the accuracy with less number of parameters.


In~\cite{resnet}, a residual architecture is proposed to construct CNNs (ResNets). ResNets enable us to train considerably deeper CNNs. The deepest ResNets employed for classification using the ILSVRC12~\cite{ILSVRC15} and Cifar datasets have 200 and 1000 layers, respectively~\cite{resnet,resnetpreact}, and achieved impressive performance. 
However, a deep ResNet with many layers may suffer from two problems. First, the running speed on parallel computing devices may be slow compared to the speed of shallower ones. Thus, it takes more time to train a very deep ResNet. Also, the gradient vanishing problem~\cite{gradvanish2,gradvanish1} may also be observed on a very deep ResNet with many layers. 
To address these problems, a novel training procedure called stochastic depth is introduced in~\cite{stochasticdepth}. It enables one to train shallow ResNets during training and use deep ResNets for testing. Shallow ResNets can ease the gradient vanishing problem during training, and reduce the training time. Their experimental results show that their proposed training procedure can improve the test accuracy of baseline ResNets with constant depth. This indicates easing the gradient vanishing problem is crucial to train a very deep ResNet. However, during the test time, the depth of ResNets is the same as the depth of baseline which makes inference speed slower. 

Intuitively, a simpler way to avoid gradient vanishing problem and accelerate the running speed during training and inference is to use a shallow network. 
In~\cite{wrn}, Zagoruyko and Komadakis use a shallow ResNet, and increase the width to make the expressive capacity comparable with narrow-deep ones. Interestingly, they observe that wide-shallow CNNs can outperform its deeper-narrower peer CNNs which have the same parameter size on the Cifar-10/100 classification datasets. On the ILSVRC12 classification task, wide ResNets can also obtain comparable accuracy with smaller depth. Their results draw our attention to consider the wide-shallow DNNs.
In their method, the width of each layer is symmetrically increased by the same factor. This will significantly increase the parameter size. Therefore, in this work, the proposed binary architecture truncate architecture of wide networks by reducing  width of the networks considering their parameter size and accuracy trade-off. Our motivation for employment of binary tree architecture also considers gradient vanishing problem, which has been shown to be crucial for training DNNs in~\cite{stochasticdepth}. 

\section{Binary Tree Architecture}




In this section, we give an overview of conventional blocks used in previous works~\cite{resnet,wrn}, and introduce our proposed binary tree architecture. Then, we theoretically analyze expressive capacity and parameter size of our proposed binary tree architecture considering its depth and width.

\subsection{Conventional Blocks} \label{cs}
In some CNNs~\cite{vgg,resnet,resnetpreact,wrn}, a block is constructed by a stack of $K$ convolutional layers. At each $k^{th}$ layer, $k=1,2,\dots, K$, we compute
\begin{equation}
\begin{split}
 \mathbf{X}_k= f_k(\mathbf{X}_{k-1};\mathcal{W}_k)
\end{split},
\label{eq:mlp}
\end{equation}
where ${\mathbf{X}_{0}:=\mathbf{X} \in \mathbb{R}^{w\times h\times c}}$. $\mathbf{X}$ is an input  tensor of features given to the block, $\mathbf{X}_k\in \mathbb{R}^{\frac{w}{s}\times \frac{h}{s} \times D}$ is the tensor of features obtained at the output of $k^{th}$ layer, $c$ is the number of channels, $w$ is the $width$ and $h$ is the $height$ of a feature map. We assume that the number of convolutional filters used at each layer is $D$, and down-sampling is performed at the first convolutional layer by stride $s$. $f_k(\mathbf{X}_k;\mathcal{W}_k)$ is computed by a composition of a convolution operation, batch normalization~\cite{bn} and nonlinear function $\sigma$, where $ \mathcal{W}_k$ denotes a set of trainable parameters used at the $k^{th}$ layer, e.g. the filter weights. The output feature tensor of the block is obtained at the last layer ${\mathbf{Y}:=\mathbf{X}_K \in \mathbb{R}^{\frac{w}{s}\times \frac{h}{s} \times D}}$. 
 In \eqref{eq:mlp}, we omit shortcut connection and element-wise addition used in~\cite{resnet} to simplify the notation. We refer to the block defined in \eqref{eq:mlp} as a \textit{conventional block} (ConvenBlock) with width $D$ and depth $K$ in the following sections. An illustration of ConvenBlock with width $256$ and depth $3$ is given in Figure~\ref{fig:bitnet} (left).



\subsection{Binary Tree Blocks} 



We define a \textit{binary tree block} (BitBlock) by a binary tree $\mathcal{T}=(\mathcal{V}, \mathcal{E})$, where $\mathcal{V}$ is the set of nodes residing in the block, and $\mathcal{E}$ is the set of edges that connect the nodes. The architecture of a BitBlock is determined by its width $D$ and depth $K$, where $D$ is the channel number, i.e. number of feature tensors provided by the BitBlock at the output, and $K$ is the depth of the binary tree $\mathcal{T}$, as follows. 

\begin{itemize}
    \item At the root of a BitBlock $\mathcal{T}$ denoted by $v_0 \in \mathcal{V}$, we have a feature tensor denoted by ${\mathbf{X}_{0,l}:=\mathbf{X} \in \mathbb{R}^{w\times h\times c}}$, where $\mathbf{X}$ is given as an input to the block.

 \item At each $k^{th}$ layer (level) of a BitBlock, we apply two mapping functions $f_{k,l}$ and $f_{k,r}$ to the input feature tensor $\mathbf{X}_{k-1,l}$. Then, we compute a feature tensor $\mathbf{X}_{k,l}$ having $ \frac{D}{2^k}$ channels on a left child node of $v_{k-1}$, and a feature tensor $\mathbf{X}_{k,r}$ having $ \frac{D}{2^k}$ channels on a right child node of $v_{k-1}$ at the $k^{th}$ layer of $\mathcal{T}$. The feature tensor $\mathbf{X}_{k,l}$ is further fed to the next $k+1^{st}$ layer. More precisely, at each $k^{th}$ level, we compute
\begin{equation}
\begin{split}
 \mathbf{X}_{k,l}=f_{k,l}(\mathbf{X}_{k-1,l};\mathcal{W}_{k,l}), \\
 \mathbf{X}_{k,r}=f_{k,r}(\mathbf{X}_{k-1,l};\mathcal{W}_{k,r}).
\end{split}
\label{eq:tree}
\end{equation}

\item Finally, feature tensors computed at all right child nodes at each $k^{th}$ level of $\mathcal{T}$, and the feature tensors computed at the left child node at the last $K^{th}$ level of $\mathcal{T}$ are concatenated across channels to construct the output feature tensor of the BitBlock by
\begin{equation}
\begin{split}
 \mathbf{Y}={\rm concat}(\mathbf{X}_{1,l} \bullet \mathbf{X}_{2,l} \bullet \dots \bullet \mathbf{X}_{K,l} \bullet \mathbf{X}_{K,r}) 
\end{split}.
\label{eq:conc}
\end{equation}
The size of concatenated tensor $\mathbf{Y}$ is $(\frac{w}{s}\times \frac{h}{s} \times D)$. In the case of down-sampling, we apply a stride $s$ at the first layer as utilized in ConvenBlocks.

\end{itemize}

We note that $D$ is divided by $2^K$. Moreover, if ${K=1}$, then the BitBlock $\mathcal{T}$ reduces to a single convolution layer. Our BitBlocks can also be applied to fully connected layers using a feature tensor $\mathbf{X} \in \mathbb{R} ^{1\times 1 \times C}$. We denote it by fully connected BitBlock. An illustration of proposed BitBlock with width $256$ and depth $3$ is given in Figure~\ref{fig:bitnet} (right). We can construct a DNN by stacking multiple BitBlocks.
A DNN endowed with BitBlocks is called as a BitNet in this paper.

\subsection{A Theoretical Analyses of Expressive Capacity of Binary Tree Blocks}\label{theory}

In this section, we theoretically analyze expressive capacity and parameter size of the proposed binary tree block with respect to its depth and width when it is used at fully connected layers. We use the definition of expressive capacity proposed in the previous work \cite{bengio1,bengio12}. Precisely, expressive capacity of a DNN is defined by the maximal number of linear regions of (decision) functions computable by the given DNN. 
The formal definition of linear regions of a function is given as follows.

\begin{definition}[\textbf{Linear Region}]
Given a function $f(\cdot)$ with $n$-dimensional input space $\mathbb{R}^n$, a linear region $\mathbb{R}_i^n$ of $f(\cdot)$ is a subspace of its input space, such that $f(\cdot)$ computes a linear mapping on that linear region, i.e. $\forall \mathbf{x}\in \mathbb{R}_i^n$,
$f(\mathbf{x}):=\mathbf{w}_i\mathbf{x}+\mathbf{b}_i$ and $\forall \mathbf{x}\notin \mathbb{R}_i^n$,
$f(\mathbf{x}):\neq\mathbf{w}_i\mathbf{x}+\mathbf{b}_i$.


\end{definition}

Consider a decision function $f(\cdot)$ computed by a DNN, then the more number of linear regions $f(\cdot)$ has, the more complex input-output mapping the DNN can compute. In other words, the DNN can compute more complex decision boundary and solve more complex classification tasks.
A multilayer perceptron network implementing a linear activation function computes a linear mapping between input and output. Thus, it only has $1$ linear region, i.e. the whole input space.



We provide the following two theoretical results regarding i) computation of the parameter size, and ii) computation of the maximal number of linear regions of functions computable by a fully connected conventional network and a BitNet. Proofs of the theorems are given in the supplemental material. We first provide the results for conventional networks that can be easily derived using the results given in \cite{bengio12}.




\begin{corr}\label{corr3}
Suppose that we are given a fully connected neural network stacked by $L$ fully connected ConvenBlocks each having width $D$ and depth $K$. The parameter size of the given network is $\mathcal{O}\big(LKD^2\big)$. The maximal number of linear regions of functions that can be computed by the given network in an $n$-dimensional ($D\geq n$) input space is lower bounded by $\mathcal{O}\big((\frac{D}{n})^{nKL} \big)$.
\QEDbs
\end{corr}

The expressive capacity of our proposed BitNet is given as follows.
\begin{prop}\label{prop3}
Suppose that we are given a BitNet stacked by $L$ fully connected BitBlocks each having width $D$ and depth $K$. The parameter size of the given BitNet is $\mathcal{O}\Big( \frac{4}{3}L\big(1-\frac{1}{4^K}\big)D^2\Big)$. The maximal number of linear regions of functions that can be computed by the given BitNet in an $n$-dimensional ($D\geq 2^Kn$) input space is lower bounded by $\mathcal{O}\Big( \big(\frac{D}{2^Kn}\big)^{nKL} \Big)$.
\QEDbs
\end{prop}

As we observe from these results, when $D$ and $L$ are fixed, as $K$ increases, the expressive capacity of BitNets can grow with a small increase of the parameter size. Although the expressive capacity of a fully connected conventional network can grow faster as $K$ increases, its parameter size also grows faster than that of a BitNet.
Although these theoretical results are obtained for fully connected layers, our experimental observations reflect that the results can be applied also for the convolutional layers. 
In our experiments, a binary tree used at convolutional layers of a DNN can increase classification accuracy with a small increase of  parameter size of the DNN.

\section{Experimental Results}

We empirically analyze the proposed binary tree architecture using various baseline ResNets and several benchmark image classification datasets. We also analyze the gradient vanishing problem of the proposed binary tree architecture. In the implementation of BitNets, given a baseline ResNet, we can construct a BitNet of the same depth and the same block width with considerably fewer parameters compared to the baseline (ResNet). We can also increase the block width but keep the depth of baseline model by the proposed binary tree architecture and obtain a BitNet with a similar number of parameters. The configuration details of BitNets are given for each classification task.

\subsection{Cifar-10 and Cifar-100}

Cifar-10 and Cifar-100~\cite{cifar} are image datasets consisting of 50,000 training images and 10,000 test images. The spatial size of each image is $32\times 32$. Cifar-10 and Cifar-100 consist of 10 and 100 categories, respectively. The architectures of BitNets and ResNets used to perform analyses on the Cifar datasets are given in Table~\ref{table:cifarnet}. As we can see from the table, the depth of a net is determined by the number of blocks in each group, i.e. $n$, and the depth of each block, i.e. $k$. The width of each block is determined by $d$.

We first compare performance of our proposed BitNets and that of Wide ResNets~\cite{wrn} by constructing nets with same depth and block width. Specifically, given a Wide ResNet with fixed value of $d$, $k$, and $n$, we use the same block width $d$ for BitNet. Then, we set values of $k$ and $n$ for BitNet, such that the total depth of BitNet equals to the total depth of the given Wide ResNet. We can use different value combinations of $k$ and $n$ for BitNet as long as $k\times n$ value for BitNet equals $k\times n$ value for the given Wide ResNet.
For example, given a $38$-layer Wide ResNet with $(d=4,k=2,n=6)$, we can construct a BitNet with $(d=4,k=3,n=4)$ or a BitNet with $(d=4,k=4,n=3)$, both implying $38$ layers.

For fare comparison with Wide ResNets, we use the same training and testing setting as employed in~\cite{wrn}.
Specifically, 
for data augmentation, 
we use a common method used in previous works~\cite{Maxout,NIN,resnet,wrn}. More precisely, 4 pixels with zero values are padded on each side of the original image to make a $40\times 40$ image, from which a $32\times 32$ patch is randomly cropped and randomly flipped horizontally. 
For testing, the original $32\times 32$ image is used. Batch size is 128 that is split on two GPUs. The initial learning rate is $0.1$, and is reduced by $0.2$ on the $60^{th}$, $120^{th}$ and $160^{th}$ epoch. The training is finished at the $200^{th}$ epoch.

\begin{table}[t]
\renewcommand{\arraystretch}{1.2}
\begin{center}
\begin{tabular}{|l|c|c|c|c|c|c|c|c|c|}
\hline
 Group Name &Configuration &Output Size\\
 \hline
 \hline
conv1       &conv, 16 channels& $32\times 32$ \\
\hline
conv2 &   block$(d\times 16,k)\times n$ &  $32\times 32$ \\
\hline
conv3($\downarrow$) & block$(d\times 32, k )\times n$ & $16\times 16$ \\
\hline
conv4($\downarrow$) &block$(d\times 64, k)\times n$ & $8\times 8$ \\
\hline

gap&  global average pooling &$1\times 1$ \\
\hline
fc&  $10$ or $100$-way softmax & \\
\hline
\end{tabular}
\end{center}
\caption{The CNN architecture employed for classification using the Cifar-10 and Cifar-100. BitNets use BitBlock as block type, and ResNets use ConvenBlock with residual connection. $k$ denotes the depth of each block ($k=2$ for all ResNets used in this paper). $d$ determines the width of each block. $n$ denotes a stack of $n$ blocks. All convolutional layers use filters of size $3\times 3$. Batch Normalization is used at every convolutional layer before ReLU. Down-sampling is performed by applying stride $2$ at the first convolutional layer of the first block in Group conv3 and conv4.}
\label{table:cifarnet}
\end{table}

Table~\ref{table:cifarerr} provides the comparative results for several Wide ResNets and BitNets, which are designed with the same width and depth. As we can see from the results, using the same depth and width, BitNet has considerably less number of parameters and FLOPs. Moreover, BitNets can outperform Wide ResNets using considerably less number of parameters. 
We compare BitNets with four baseline Wide ResNets. In the analyses, we obtained the following results:

\begin{table*}[t]
\begin{center}
\small
\renewcommand{\arraystretch}{1.2}
\begin{tabular}{ l|c c c|c c}
\hline
 Model &Depth&Param&FLOP & Cifar10& Cifar100  \\
  \hline
 NIN~\cite{NIN} &-& -&-&$8.81$  &$35.67$ \\
 \hline
 ELU~\cite{elu} &-& -&-&$6.55$  &$24.28$ \\
 \hline
 DSN~\cite{DSN} &-& -&-& $7.97$ &$34.57$ \\
 \hline
 AllCNN~\cite{Allconv} &-& -&-& $7.25$ &$33.71$ \\
 \hline
  ResNet~\cite{resnet} &$1202$& $10.2$M&-&$4.91$  &$-$ \\
 \hline
preact-ResNet~\cite{resnetpreact}    &$1001$& $10.2$M&-&$4.62$&$22.71$ \\
\hline
Stochastic Depth ResNet~\cite{stochasticdepth}    &$110$& $1.7$M&-&$ 5.25 $&$24.98$ \\
\hline
FractalNet~\cite{fractalnet}    &$40$& $22.9$M&-&$ 5.24 $&$22.49$ \\
\hline
 Wide ResNet (d=4,k=2,n=6)~\cite{wrn} &$38$& $8.9$M&$1.34\times 10^9$&$4.97$  &$22.89$ \\
 
  BitNet (d=4,k=3,n=4) &$38$& $3.7$M&$0.53\times 10^9$&$4.82$  &$\underline{22.19}$ \\
 
  BitNet (d=4,k=4,n=3) &$38$& $2.7$M&$0.39\times 10^9$&$\underline{4.65}$  &$22.60$ \\

  BitNet (d=4,k=2,n=6) &$38$& $5.4$M&$0.78\times 10^9$&$5.31$  &$23.22$ \\

  BitNet (d=4,k=6,n=2) &$38$& $1.7$M&$0.24\times 10^9$&$4.77$  &$23.87$ \\
  \hline
  Wide ResNet (d=10,k=2,n=2)~\cite{wrn} &$14$& $17.1$M&$2.64\times 10^9$&$4.56$  &$21.59$ \\
 
   BitNet (d=10,k=2,n=2) &$14$& $9.6$M&$1.32\times 10^9$& $\underline{4.17}$  &$\underline{20.48}$ \\
 
   BitNet (d=10,k=4,n=1) &$14$& $3.9$M&$0.49\times 10^9$&$4.97$  &$23.88$ \\
  \hline
  Wide ResNet (d=10,k=2,n=3)~\cite{wrn} &$20$& $26.8$M&$4.06\times 10^9$&$4.44$  &$20.75$ \\
 
   BitNet (d=10,k=2,n=3) &$20$& $15.6$M&$2.21\times 10^9$&$\underline{\mathbf{3.78}}$  &$\underline{19.29}$ \\
 
    BitNet (d=10,k=3,n=2) &$20$& $10.2$M&$1.41\times 10^9$&$3.81$  &$19.37$ \\
  \hline
  Wide ResNet (d=12,k=2,n=4)~\cite{wrn} &$26$& $52.5$M&$7.87\times 10^9$&$4.33$  &$20.43$ \\
 
   BitNet (d=12,k=2,n=4) &$26$& $31.2$M&$4.45\times 10^9$&$\underline{4.07}$  &$\underline{\mathbf{19.06}}$\\
  
   BitNet (d=12,k=4,n=2) &$26$& $14.9$M&$2.06\times 10^9$&$4.11$  &$19.22$ \\  
  \hline

\end{tabular}
\end{center}
\caption{Classification error (\%) of CNNs for the Cifar-10/100 datasets. Using the same depth and block width, our BitNets can outperform Wide ResNets with considerably less number of parameter size. 
Underlined numbers indicate the best performance among models having the same depth and same block width. Bold numbers denote the best performance obtained for all models.
Definition of $d$, $k$ and $n$ are given in Table~\ref{table:cifarnet}. All of Wide ResNets and our BitNets are trained using data augmentation and without using dropout.}
\label{table:cifarerr}
\end{table*}

(1) A Wide ResNet having $(d=4,k=2,n=6)$ is the deepest and narrowest architecture among four baseline architectures. 
 The BitNet $(d=4,k=4,n=3)$ and the BitNet $(d=4,k=3,n=4)$ can obtain comparable performance with this baseline and the parameter size is only $30\%$ and $41\%$ of that of baseline, respectively. The BitNet $(d=4,k=2,n=6)$ is constructed by using more number of blocks but reduce the depth of each block. As shown in Section \ref{gv}, this BitNet suffers from a gradient vanishing problem during the training due to use of small $k$ and a large $n$.
 As a result, the performance slightly degrades. BitNet $(d=4,k=6,n=2)$ has least parameter size, i.e. only $19\%$ of baseline's parameter size. However, it also performs worst.
 
 (2) Compared with first baseline ResNet $(d=4,k=2,n=6)$, the baseline ResNet $(d=10,k=2,n=2)$ is wider and shallower.
 The BitNet $(d=10,k=2,n=2)$ outperforms it by $1\%$ in the Cifar100 task by using approximately $56\%$ of baseline ResNet's parameter size. Another configuration of BitNet $(d=10,k=4,n=1)$ uses only one BitBlock at each group, resulting in only $23\%$ of baseline's parameter size. However, the performance is also degraded by more than $1\%$ using the Cifar-100 dataset. For the Cifar-10 dataset, both BitNets obtain similar performance compared to the baseline.
 
(3) For the Wide ResNet $(d=10,k=2,n=3)$, both BitNets obtain more than $1\%$ performance boost using the Cifar-100 dataset. For the Cifar-10 dataset, the performance boost is more than $0.5\%$. Notably, the parameter size of the BitNet $(d=10,k=3,n=2)$
is only $38\%$ of the parameter size of the baseline.

(4) The Wide ResNet $(d=12,k=2,n=4)$ has the largest parameter size among four baseline models. Our two BitNets both outperform the baseline by more than $1\%$ accuracy. Note that the parameter size of the BitNet $(d=12,k=4,n=2)$ is only $28\%$ of that of the baseline.

We emphasize that the performance of baseline Wide ResNets are already close to the state-of-the-art. Thus more than $1\%$ boost of the accuracy is obtained. To summarize, most of our BitNets can achieve better or approximately equal accuracy using less number of parameters, which indicates that our binary tree architecture can improve the parameter size and accuracy trade-off of baseline Wide ResNets. There are two BitNets whose accuracy are roughly $1\%$ lower than that of their baselines. This is possibly because they use too less number of BitBlocks causing insufficient expressive capacity. The rest of BitNets obtain sufficient expressive capacity with a relatively less number of parameter size.
Compared with other previous models such as Stochastic Depth ResNet~\cite{stochasticdepth} and FractalNet~\cite{fractalnet}, our BitNets can outperform them using less number of parameter size.

\subsection{ILSVRC12}

To evaluate the proposed architecture on a large scale image classification task, we also use the training and validation dataset of ILSVRC12~\cite{ILSVRC15}, which consists of 1.3M training images and 50,000 validation images belonging to 1000 categories. 
During training, data augmentation and image preprocessing methods are used as follows. The image is cropped by scale and aspect ratio augmentation method~\cite{googlenet}, and then resized to $224\times 224$. A random horizontal flip is also applied. The input images are mean subtracted and variance normalized on each RGB channel. The color distortion methods proposed in~\cite{alexnet} and~\cite{colordist} are both used. For validation, the image is resized such that its shorter side is 256, and a center crop of $224\times 224$ are used to test. Batch size is 256  split on 8 GPUs. The initial learning rate is 0.1 and is reduced by $10^{-1}$ at each $30$ epoch. The training is finished at the $90^{th}$ epoch.
Following~\cite{resnet}, stochastic gradient descent (SGD) with momentum 0.9 is used as our optimizer and the weight decay is set as 0.0001. Batch Normalization is used in every convolutional layer before ReLU.
We didn't use dropout~\cite{dout} in any BitNet.

In this task, we construct two BitNets (BitNet-26 and BitNet-34) in order to compare their performance with that of ResNet-34~\cite{resnet}. The details of models are given in Table~\ref{table:bitnets}.
The classification results are given in Table~\ref{table:imgacc}.
The results show that our BitNets have better parameter size and accuracy trade-off than ResNets. Specifically, BitNet-34
outperforms ResNet-34 B by $1\%$ using same parameter size, while their depth is the same.
Another smaller BitNet-26 obtains $1\%$ higher error compared to ResNet-34 B. However it is shallower and the parameter size is approximately $50\%$ of ResNet-34 B. 
Compared to model FractalNet-34, BitNet-34 also outperforms it. BitNet-26 outperforms ResNet-18 B by approximately $2\%$ accuracy using the same number of parameters. Increasing the wide of ResNet-18 by methods in~\cite{wrn} can improve the accuracy. However the cost for the improvement is also huge. BitNet-26 and ResNet-18 B width $\times2$ have comparable accuracy, but the parameter size of our BitNet is almost two times smaller than that of ResNet. Similarly, the parameter size of BitNet-34 is only $50\%$ of ResNet-18 width $\times3$.

\begin{table}[t]
\small
\begin{center}
\renewcommand{\arraystretch}{1.2}
\begin{tabular}{|l|c|c|c|}
\hline
Group &BitNet-26&BitNet-34&Output Size\\
\hline
conv1 &\multicolumn{2}{c|}{conv $7\times7$, 64 channels, stride 2}& $112\times 112$   \\
\hline
 \multirow{2}{*}{conv2}&\multicolumn{2}{c|}{max pooling $3\times3$, stride 2}&\multirow{2}{*}{$56\times 56$} \\
 \cline{2-3}
 &  b$(128,3)\times 2$ &b$(256,4)\times 2$& \\

\hline
 \multirow{2}{*}{conv3}&  
b$(192,3)\times 1$
 & b$(384,4)\times 2$&\multirow{2}{*}{$28\times 28$} \\
& b$(256,3)\times 1$&&\\
\hline
 conv4&
  b$(384,3)\times 2$
 & b$(512,4)\times 2$&$14\times 14$ \\
\hline
 \multirow{2}{*}{conv5}&
  b$(512,3)\times 1$
 & b$(768,4)\times 2$&\multirow{2}{*}{$7\times 7$} \\
 &  b$(768,3)\times 1$ &&\\
\hline
gap&\multicolumn{2}{c|}{global average pooling}&$1\times1$   \\
\hline
fc&\multicolumn{2}{c|}{1000-way softmax} &     \\
\hline
\hline
Param. &$12.79$M&$22.99$M&\\
\hline
FLOP  &$2.8\times10^9$   &$7.8\times10^9$& \\
\hline
Depth &$26$&$34$& \\
\hline
\end{tabular}
\end{center}
\caption{Structure of BitNets used for ILSVRC12 classification task. b$(d,k)\times n$ refers to a stack of $n$ BitBlocks with width $d$ and depth $k$. All convolutional layers employed in each BitBlock use filters of size $3\times 3$. The output size is reduced by applying stride $2$ at the first convolutional layer of first block in some Groups.}
\label{table:bitnets}
\end{table}

\begin{table}
\small
\renewcommand{\arraystretch}{1.2}
\begin{center}
\begin{tabular}{l|c|c|c}
\hline
Model & \makecell{Single\\Crop} & \makecell{Ten\\Crop} &Param\\
 \hline



ResNet-18 B~\cite{resnet}& $30.43$ &$28.22$&  $13.1$M \\
Width $\times2$~\cite{wrn} & $27.06$ &$-$&  $25.9$M\\
Width $\times3$~\cite{wrn} &$25.58$& $-$&$45.6$M\\
\hline
FractalNet-34~\cite{fractalnet}&$-$&$24.12$&$-$\\
\hline
ResNet-34 B~\cite{resnet} & $26.73$&  $24.76$ & $23.2$M\\
Width $\times2$~\cite{wrn} & $24.5$ &$-$&  $48.6$M\\
\hline
BitNet-26&$27.74$&$25.83$&$12.8$M\\
BitNet-34&$25.46$&$23.77$&$23.0$M\\
\hline

\end{tabular}
\end{center}
\caption{Single model, Top-1 classification error (\%) obtained using ILSVRC12 validation dataset.}
\label{table:imgacc}
\end{table}

\subsection{Experimental Analyses of Depth of BitBlock}
Also, we observe that using the same width $d$ and same number of BitBlocks $n$, BitBlocks having different depth may provide different performance (see BitNet $(d=10,k=3,n=2)$ and BitNet $(d=10,k=2,n=2)$ in Table~\ref{table:cifarerr}). Thus, we further analyze how performance of BitNets changes with respect to the depth of BitBlocks. Specifically, we evaluate BitNet $(d=4,k,n=4)$ (a deep-narrow one) and BitNet $(d=12,k,n=2)$ (a relatively shallower-wider one) by setting different values to $k$. The results are given in Table~\ref{table:wrtk}. As illustrated in the table, as $k$ increases, both BitNets gain a performance boost due to an increase on the expressive capacity. Note that for the BitNet $(d=12,k,n=2)$ employed using the Cifar-100, there is almost a $4\%$ performance boost from $k=1$ to $k=2$ with a $33\%$ increase of parameter size. We observe that the performance boosts further as $k$ increases. These observations match our theoretical analyses provided in Section~\ref{theory}, which states that as $d$ and $n$ are fixed, increasing $k$ can increase the expressive capacity of a BitNet with a small increase of parameter size.
However, the boosting trend tends to be saturated after $k\geq 3$, and the increase of parameter size is also neglectable. This can be explained as follows. As $k$ increases, the width of convolutional layer also decreases in the binary tree architecture resulting in a saturated expressive capacity.
Additionally, we also observe that wider BitNets ($d=12$) gain more performance boost than a narrower Bitnet ($d=4$) with the same increase on $k$. This is simply because the increased layers in the wider BitNet are wider than narrower one, thus it can increase more expressive capacity.  

\begin{table}[t]
\renewcommand{\arraystretch}{1.2}
\small
\begin{center}
\begin{tabular}{l|c c|c c}
\hline
   &\multicolumn{2}{c|}{\makecell{BitNet\\$(d=4,k,n=4)$}} &\multicolumn{2}{c}{\makecell{BitNet\\$(d=12,k,n=2$)}}\\
   &Param.&Cifar10/100&Param.&Cifar10/100\\
 \hline
 
$k=1$ &$2.7$M&$4.98/23.54$ & $10.3$M  & $6.02/23.80$ \\
\hline
$k=2$ &  $3.5$M&$\mathbf{4.68}/22.72$ & $13.8$M&$4.02/19.86$  \\
\hline
$k=3$ & $3.7$M&$4.82/\mathbf{22.19}$ & $14.7$M&$\mathbf{3.98}/18.97$ \\
\hline
$k=4$ &$3.7$M&$4.69/22.65$ & $14.9$M&$4.11/19.22$ \\
\hline
$k=5$ &$3.7$M&$4.69/22.86$ & $15.0$M&$4.09/\mathbf{18.95}$ \\
\hline
$k=6$ &$3.7$M&$4.77/22.69$ & $15.0$M&$4.19/19.51$ \\
\hline
\end{tabular}
\end{center}
\caption{Cifar-10/100 classification error (\%) of two BitNets with respect to $k$. Definition of $d$, $n$, and $k$ are given in Table~\ref{table:cifarnet}.}
\label{table:wrtk}
\end{table}

\subsection{Experimental Analyses of Gradient Vanishing Problem}\label{gv}

In this section, we analyze our BitNets considering the gradient vanishing problem. Specifically, during the training of a BitNet or Wide ResNet employed using the Cifar-100, we compute the mean magnitude (L2-norm) of gradient obtained at the first convolutional layer per epoch.
The results are given in Figure~\ref{fig:grads}.
We analyze 9 models having different depth but same width $d=4$.
In the figure, wide resnet-38 refers to the Wide ResNet $(d=4,k=2,n=6)$ having  $38$ layers and plainnet-38 is the  one designed without using residual shortcut connections. Models denoted by b$(d,k,n)-m$ are proposed BitNets and $m$ is the total depth.

\begin{figure}[t]
\begin{center}
   \includegraphics[width=0.9\linewidth]{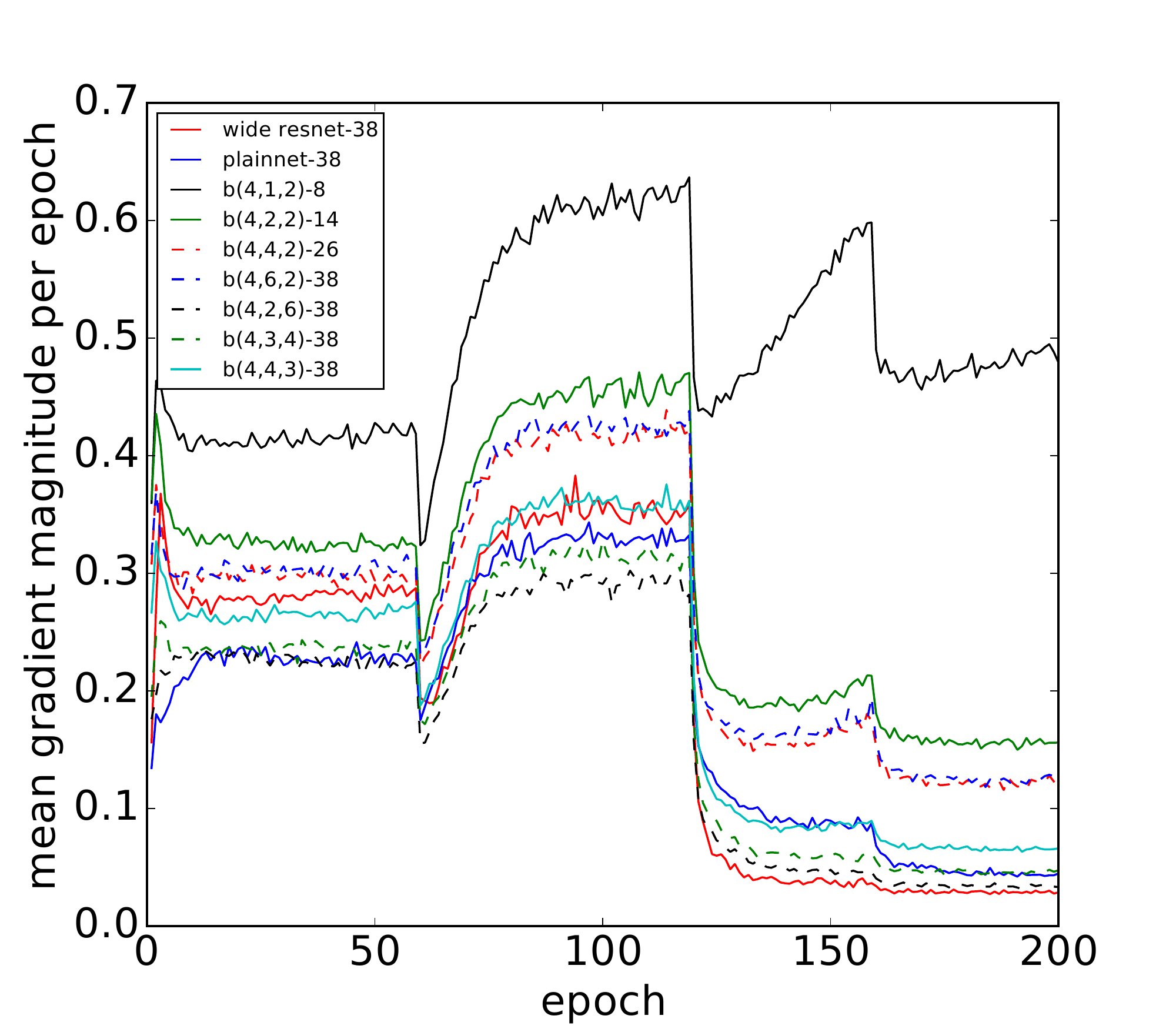}
\end{center}
   \caption{Mean magnitude (L2-norm) of gradient of the first convolutional layer computed  per epoch during the training. b$(d,k,n)-m$ refers to an m-layer BitNet with structure configuration $(d,k,n)$ defined in Table~\ref{table:cifarnet}. Best viewed in color print.}
\label{fig:grads}
\end{figure}

As we can see from the figure, for all models having $38$ layers, our BitNet b$(4,6,2)-38$ shows the strongest magnitude, even stronger than wide resnet-38. This results indicates that using concatenation in the proposed binary tree architecture can ease the gradient vanishing problem. We also observe that gradient becomes weaker as the number of blocks $n$ is increased and the depth $k$ of each BitBlock is reduced. For instance, for all $38$ layers BitNets, the magnitude can be roughly sorted by 
b$(4,2,6)-38$ weaker than b$(4,3,4)-38$ weaker than b$(4,4,3)-38$ weaker than b$(4,6,2)-38$ according to the strength of the magnitude.
This observation reflects that as $n$ increases and $k$ decreases, features obtained at less number of lower layers are concatenated to form the output of each BitBlock. In general, the gradients  propagate more layers to reach lower layers. 

We also analyze how the gradient magnitude changes with respect to BitBlock's depth $k$ if $d$ and $n$ are fixed. As we can see, b$(4,1,2)-8$ shows the strongest magnitude of gradient among all nine models as expected because it is the shallowest model. By increasing $k$, we observe that for BitNet b$(4,2,2)-14$ and b$(4,4,2)-26$, the magnitude is decreased because more layers  are used to propagate gradient in BitBlocks.

The errors obtained for nine models are given in Table~\ref{table:graderror}.
Although b$(4,1,2)-8$ and b$(4,2,2)-14$ show stronger gradient magnitude than wide resnet-38, they provide higher  classification errors. This is mainly because the depth of these two BitNets is too small resulting in insufficient expressive capacity. BitNet b$(4,4,3)-38$ obtains comparable classification performance by using larger depth. The gradient magnitude of BitNet b$(4,4,3)-38$ is also comparable with that of resnet, which benefits from the proposed binary tree architecture. Without using binary tree architecture, the gradient magnitude of plainnet-38 is weaker than that of resnet-38 and the final classification error is larger.

\begin{table}[t]
\renewcommand{\arraystretch}{1.2}
\begin{center}
\begin{tabular}{l|c|c}
\hline
  
 Model  &Param.&Error\\
 \hline
 
wide resnet-38 &$8.9$M&$ 22.89 $ \\
\hline
plainnet-38   &$8.9$M&$ 29.13 $\\
\hline
b$(4,1,2)-8$  &$1.2$M&$ 29.19 $ \\
\hline
b$(4,2,2)-14$ &$1.5$M&$ 25.45 $ \\
\hline
b$(4,4,2)-26$ &$1.7$M&$ 22.72 $ \\
\hline
b$(4,6,2)-38$ &$1.7$M&$ 23.87 $ \\
\hline
b$(4,2,6)-38$ &$5.4$M&$ 23.22 $ \\
\hline
b$(4,3,4)-38$ &$3.7$M&$ 22.19 $ \\
\hline
b$(4,4,3)-38$ &$2.7$M&$ 22.60 $ \\
\hline
\end{tabular}
\end{center}
\caption{Cifar-100 classification error (\%) of the models illustrated in Figure~\ref{fig:grads}.}
\label{table:graderror}
\end{table}




\section{Conclusions}

In this paper, we introduce and analyze a binary tree architecture to truncate architecture of wide networks considering their parameter size and accuracy trade-off.
In the proposed architecture, the width at each layer is continuously reduced from lower layers to higher layers.
Also, features obtained at different layers are concatenated to form the output of our architecture. 
In our experiments, the networks which are designed using the proposed architecture, called BitNets, can obtain better parameter size and accuracy trade-off on several benchmark datasets compared to baseline networks endowed with conventional architectures.
Additionally, in our experimental analyses, we observe that the concatenation structure can ease the gradient vanishing problem.
We also provide a theoretical analyses of the expressive capacity of BitNets. 
In our future work, we plan to use BitNets for object detection tasks.

{\small
\bibliographystyle{ieee}

}

\raggedbottom
\pagebreak

\section*{Supplemental Material}

In the supplemental material, we provide the proofs of Corollary \ref{corr3} and Proposition \ref{prop3} in Section \ref{theory} in the main paper. We also show training errors of proposed BitNets used in our experiments. 

\section*{A.  Proofs}

\begin{proof}[Proof of Corollary \ref{corr3}]

The total depth of the given network is $LK$. Thus, the parameter size is $\mathcal{O}\big(LKD^2\big)$. Also, according to Corollary 5 of~\cite{bengio12}, the maximal number of linear region of functions that can be computed by the given network in an $n$-dimensional ($D>n$) input space is lower bounded by $\mathcal{O}\big((\frac{D}{n})^{n(KL-1)} D^n \big)$, which can be simplified by a looser bound $\mathcal{O}\big((\frac{D}{n})^{nKL} \big)$ and resulting in our corollary.

\end{proof}

\begin{proof}[Proof of Proposition \ref{prop3}]
According to the definition of BitBlock, the width of $k^{th}$ layer in a BitBlock is $\frac{D}{2^{(k-1)}}$. Thus, the parameter size of the given BitBlock can be computed by,
\begin{equation*}
\begin{split}
 &\sum \limits_{k=1}^{K} \Big ( \frac{D}{2^{(k-1)}} \Big )^2 \\
 =&\frac{4}{3}\big(1-\frac{1}{4^K} \big)D^2
 \end{split}.
\label{eq:sp}
\end{equation*}
Thus, the parameter size of $L$ stacked BitBlocks is $\mathcal{O}\Big( \frac{4}{3}L\big(1-\frac{1}{4^K}\big)D^2\Big)$.
According to Theorem 4 of~\cite{bengio12}, the maximal number of linear regions of functions that can be computed by a given BitBlock in an $n$-dimensional ($D\geq 2^Kn$) input space is lower bounded by
\begin{equation*}
\begin{split}
 &\prod \limits_{k=1}^{K} \Big ( \frac{D}{2^{k}n} \Big )^n \\
 =& \Big ( \frac{D}{n} \Big )^{nK} 2^{\frac{-n(1+K)K}{2}}\\
 >& \Big ( \frac{D}{n} \Big )^{nK} 2^{-nKK} \\
 =&\Big ( \frac{D}{2^Kn} \Big )^{nK} .
 \end{split}
\label{eq:sp2}
\end{equation*}

\end{proof}
As a result, with $L$ BitBlocks, the maximal number of linear region is bounded by $\mathcal{O}\Big( \big(\frac{D}{2^Kn}\big)^{nKL} \Big)$.

\section*{B. Training Errors}
 Training and testing errors of BitNets used in our experiments for Cifar-100 classification task are given in Table~\ref{table:trainerr}. 
 As shown in~\cite{bengio1,bengio12}, the expressive capacity reflects the complexity of class decision boundary computable by a DNN. We use training errors to quantify the expressive capacity of a DNN.
 As we can see from the table, the training errors of BitNets are close to that of the Wide ResNets, which indicates that by using considerably less number of parameters, the proposed BitNets can obtain expressive capacity similar to that of the Wide ResNets.


\begin{table}
\renewcommand{\arraystretch}{1.2}
\begin{center}
\begin{tabular}{l|c|c|c}
\hline
  
 Model  &Param.&Test Err.&Train Err.\\
 \hline

\hline
Wide ResNet (d=4,k=2,n=6) &$8.9$M&$ 22.89 $&$0.018$\\
BitNet (d=4,k=3,n=4)  &$3.7$M&$ 22.19 $&$0.018 $ \\
BitNet (d=4,k=4,n=3)  &$2.7$M&$ 22.60 $&$0.022 $ \\
BitNet (d=4,k=2,n=6)  &$5.4$M&$ 23.22 $&$0.026 $ \\
BitNet (d=4,k=6,n=2)   &$1.7$M&$ 23.87 $&$ 0.028 $ \\
\hline
Wide ResNet (d=10,k=2,n=2) &$17.1$M&$ 21.59 $&$0.018$\\
BitNet (d=10,k=2,n=2)   &$9.6$M&$ 20.48 $&$0.018$ \\
BitNet (d=10,k=4,n=1)  &$3.9$M&$ 23.88 $&$0.020$ \\
\hline
Wide ResNet (d=10,k=2,n=3) &$26.8$M&$ 20.75 $&$0.018$\\
BitNet (d=10,k=2,n=3)  &$15.6$M&$ 19.29 $&$0.018$ \\
BitNet (d=10,k=3,n=2)  &$10.2$M&$ 19.37 $&$0.018$ \\
\hline
Wide ResNet (d=12,k=2,n=4) &$52.5$M&$ 20.43 $&$0.018$\\
BitNet (d=12,k=2,n=4)  &$31.2$M&$ 19.06 $&$0.018$ \\
BitNet (d=12,k=4,n=2)  &$14.9$M&$ 19.22 $&$0.018$ \\

\hline
\end{tabular}
\end{center}
\caption{Cifar-100 classification error (\%) of the BitNets and Wide ResNets used in our experiments.}
\label{table:trainerr}
\end{table}

\end{document}